%% file: main.tex
\pdfoutput=1
\documentclass{article}
\usepackage{icomp2026_conference,times}

\input{math_commands.tex}

\usepackage{hyperref}
\usepackage{url}
\usepackage{amsmath}
\usepackage{amssymb}
\usepackage{amsthm}
\usepackage{booktabs}
\usepackage{array}
\usepackage{graphicx}
\usepackage{tikz}
\usetikzlibrary{arrows.meta,positioning}
\usepackage{cleveref}
\newtheorem{proposition}{Proposition}
\newtheorem{corollary}{Corollary}

\title{TempShap: When and How Much a Feature Matters in Time-Series Forecasting}

\author{%
\hspace*{-\tabcolsep}%
\makebox[\textwidth][c]{%
\begin{tabular}{@{}c@{\hspace{2.2em}}c@{}}
Chi Zhao
&
Anastasiia Chelnokova
\\[-0.15em]
{\small\normalfont\texttt{chi.zhao@spbu.ru}}
&
{\small\normalfont\texttt{st133927@student.spbu.ru}}
\\[0.65em]
Jing Liu
&
Elena Parilina
\\[-0.15em]
{\small\normalfont\texttt{st082130@student.spbu.ru}}
&
{\small\normalfont\texttt{e.parilina@spbu.ru}}
\\[0.85em]
\multicolumn{2}{c}{%
  \normalfont Saint Petersburg State University%
}
\\
\multicolumn{2}{c}{%
  \small\normalfont
  7/9 Universitetskaya nab., Saint Petersburg 199034, Russia%
}
\end{tabular}%
}%
\hspace*{-\tabcolsep}%
}

\icompfinalcopy

\begin{document}

\maketitle
\fancyhead{}\renewcommand{\headrulewidth}{0pt}

\begin{abstract}
Attribution for a time-series forecaster should answer two questions that a
feature-level scalar cannot: when in the input window a feature acts, and how
much its influence changes across the forecast horizon. Most existing Shapley
explainers treat each feature as one player and therefore collapse both axes.
We propose \textit{TempShap}, a retrospective, loss-based Shapley framework whose
players are the individual input cells $(t,i)$. Its attribution tensor yields temporal-localization and horizon-impact profiles. We validate these against planted temporal structures,
a tractable exact cell game, and a closed-form attribution oracle for a real
linear forecaster. Permutation sharing estimates the cell-level game at localized perturbation
cost. Under redundant inputs, one-at-a-time scores under-credit substitutable
information, whereas the signed Shapley decomposition preserves total value and
yields more reliable held-out deletion rankings
under mean masking. A planted counterexample isolates the redundancy mechanism.
Experiments show a deletion-ranking advantage across Shapley estimators,
forecasting benchmarks, and model families. \textit{TempShap} therefore makes
fine-grained Shapley attribution both practical and operationally useful.
\end{abstract}

\section{Introduction}
\label{sec:intro}

Two questions a practitioner cannot avoid when interpreting a
multivariate forecaster: \emph{when} in the input window does a feature
exert its influence, and \emph{how much} does that influence change
across the prediction horizon. Yet most existing Shapley explainers
(\textit{SHAP} \citep{lundberg2017unified}, \textit{KernelSHAP}, \textit{FastSHAP}
\citep{jethani2022fastshap}, \textit{Leverage SHAP} \citep{musco2025leverage},
\textit{ShapG} \citep{zhao2025shapg}) treat each feature as one player, returning
one scalar per input feature. This collapses the input-lag and the whole
output cell (forecast horizon and target) to that scalar. The standard
evaluation does the same:
estimators are ranked by a progressive feature-removal score
\citep{petsiuk2018rise}, a single curve that again collapses the input-lag
and horizon axes. \textit{ShapTime} \citep{zhang2024shaptime} explains temporal
importance to guide forecasting improvements. Other time-series extensions push
toward $(\text{feature},\text{timestep})$ space by \emph{grouping}
(\textit{TimeSHAP} \citep{bento2021timeshap}, \textit{WindowSHAP}
\citep{nayebi2023windowshap}, \textit{ShaTS} \citep{franco2025shats}); finer-grained
approaches use cell-level Shapley perturbations (\textit{TimeSHAP}) or localized
per-position sensitivities (\textit{PAX-TS}), but neither provides a \emph{uniform
loss-based} per-(sample, input-lag, input-feature, horizon, target) Shapley
tensor from which temporal profiles are read; that tensor is this paper's
contribution.

Our contributions are:
\begin{itemize}
  \item \textit{A cell-level, loss-based Shapley framework}
  (\Cref{sec:method}) whose attribution tensor yields temporal-localization,
  horizon-impact, and rank-settling profiles (\Cref{fig:schematic}).
  \item \textit{Validation and robustness of this framework}
  (\Cref{sec:results-synth,app:sensitivity}): recovery of planted signals,
  exact-cell and linear-model oracle agreement, baseline-robust localization,
  and no consistent recency ordering across architectures.
  \item \textit{A signed, efficiency-respecting advantage under redundancy}
  (\Cref{sec:results-efficiency}): accounting for substitutable inputs
  improves held-out cell-deletion rankings across estimators
  under mean masking.
  \item \textit{A practical, unbiased estimator}
  (\Cref{sec:results-competitor}): permutation sharing delivers axiomatic
  cell attribution at localized perturbation cost.
\end{itemize}

\section{Related work}
\label{sec:related}

\paragraph{Shapley attribution for tabular models.}
Permutation-based Monte Carlo estimation of the Shapley value predates
\textit{SHAP}: \citet{castro2009polynomial} give a polynomial-time estimator for
general cooperative games. In machine learning, Shapley sampling values
adapt permutation-based estimation to feature attribution
\citep{strumbelj2014explaining}. \citet{lundberg2017unified} subsequently
unified these and other additive attribution methods under \textit{SHAP} and
introduced \textit{KernelSHAP}. Later work improves Shapley estimation through
paired \textit{KernelSHAP} sampling \citep{covert2021kernelshap}, amortization
\citep{jethani2022fastshap}, leverage-score sampling
\citep{musco2025leverage}, and graph-neighborhood sampling
\citep{zhao2025shapg}.
The \texttt{shapr} package \citep{jullum2025shapr} estimates
\emph{conditional} (observational) Shapley values, in contrast to the
interventional (marginal) game used in our main analysis: the ``true to
the data'' versus ``true to the model'' distinction \citep{chen2020true}.
Finally, \citet{covert2020understanding} introduce \textit{SAGE}, a loss-based
global Shapley importance measure. We adopt the loss-based game while
retaining the sample, forecast-horizon, and target axes.

\paragraph{Shapley attribution for time series.}
\citet{bento2021timeshap} (\textit{TimeSHAP}) combine feature-, timestep-, and
cell-level (event$\times$feature) \textit{KernelSHAP} perturbations with sequence
pruning for recurrent classifiers; \citet{nayebi2023windowshap}
group adjacent timesteps into windows; \citet{franco2025shats} use
per-(feature, segment) groups, and \citet{kim2026gsshap}
(\textit{GroupSegment-SHAP}) make cross-variable groups over time segments the
players. \citet{raykar2023tsshap} (\textit{TsSHAP}) explain forecasters via \textit{TreeSHAP} on
an interpretable surrogate; \citet{kreuzer2025paxts} (\textit{PAX-TS}) give
model-agnostic multi-granular explanations through localized
perturbations. \citet{villani2022varshap} present closed-form \textit{SHAP} for
VARMAX-class models and \emph{Time Consistent SHAP}, instantiating the
dynamic-game time-consistency of \citet{petrosjan2003time} via
separately re-computed prefix sub-games. These deliver
horizon-specific, time-segment, or multi-granular explanations;
\textit{TempShap}'s contribution is the \emph{uniform loss-based per-output-cell
tensor} over $(t,i)$ players and the temporal profiles it makes
computable. \textit{Temporal-Decay Shapley} \citep{pang2026tds} is data
valuation: it scores training \emph{samples}, not (input-cell,
output-cell) pairs, and is orthogonal despite sounding related.

\paragraph{Non-Shapley attribution.}
\textit{LIME} \citep{ribeiro2016should}, \textit{Integrated Gradients}
\citep{sundararajan2017axiomatic}, and \textit{DeepLIFT} \citep{shrikumar2017learning} are general
attribution methods applicable to neural forecasters, but saliency fails
to reliably identify feature importance over time, conflating the time
and feature axes \citep{ismail2020benchmarking}. \textit{Dynamask}
\citep{crabbe2021explaining} is the closest non-Shapley analogue for
dense input-cell saliency, though it targets single-output prediction
rather than multi-step forecasting.

\paragraph{Faithfulness evaluation.}
Attribution quality is assessed with deletion-based faithfulness scores
such as deletion-AUC \citep{petsiuk2018rise,hooker2019roar} and the
weighted-slope deletion score \citep{zhao2025shapg}, together with sanity
checks \citep{adebayo2018sanity}.

\section{The TempShap framework}
\label{sec:method}

\Cref{fig:schematic} shows the TempShap workflow in three steps.
(a) We treat input cells as players, replace masked cells with their training
means, and use each output cell's negative squared error as the payoff.
(b) We estimate cell Shapley values and collect them in a tensor.
(c) We average their magnitudes to obtain temporal-localization and
horizon-impact profiles, then sum horizon impacts to rank features.

\begin{figure}[!htb]
\centering
\resizebox{\linewidth}{!}{%
\begin{tikzpicture}[
  present/.style={draw=gray!55, minimum size=0.5cm, inner sep=0pt, fill=blue!8},
  masked/.style={draw=gray!55, minimum size=0.5cm, inner sep=0pt, fill=black!45},
  lbl/.style={font=\scriptsize, align=center},
  ttl/.style={font=\small\bfseries},
  box/.style={draw, rounded corners, align=center, fill=black!3, font=\small},
  >={Stealth[length=5pt]},
]
  \foreach \ry in {0,1,2}{\foreach \cx in {0,...,5}{\node[present] at (\cx*0.5,-\ry*0.5) {};}}
  \foreach \cx in {0,...,5}{\node[masked] at (\cx*0.5,-0.5) {};}
  \draw[red!75!black,very thick,rounded corners=2pt] (-0.28,-0.24) rectangle (2.78,-0.76);
  \node[ttl] at (1.25,0.5) {Feature-level};
  \node[lbl] at (1.25,-1.5) {input lag $t\rightarrow$};
  \node[lbl,text width=3.8cm] at (1.25,-2.0)
    {player $=$ feature (all $T$ lags)\\ \textcolor{red!75!black}{\emph{which lag?} lost}};
  \draw[->,thick] (2.95,-0.5) -- (3.75,-0.5);
  \node[lbl] at (3.35,-0.12) {split into\\cells};
  \begin{scope}[shift={(4.2,0)}]
    \foreach \ry in {0,1,2}{\foreach \cx in {0,...,5}{\node[present] at (\cx*0.5,-\ry*0.5) {};}}
    \node[masked] at (1.5,-0.5) {};
    \draw[green!50!black,very thick,rounded corners=1pt] (1.24,-0.26) rectangle (1.76,-0.74);
    \node[ttl] at (1.25,0.5) {Cell-level {\footnotesize(TempShap)}};
    \node[lbl] at (1.25,-1.5) {input lag $t\rightarrow$};
    \node[lbl,text width=3.8cm] at (1.25,-2.0)
      {player $=$ cell $(t,i)$\\ \textcolor{green!50!black}{pinpoints lag $\Rightarrow$ \emph{when}}};
  \end{scope}
  \draw[->,thick] (6.95,-0.5) -- (7.5,-0.5);
  \node[box,anchor=west] (game) at (7.55,-0.5)
    {mask $X|_S$, run $f_\theta$\\$v(S)=-(Y-\hat Y|_S)^2$};
  \begin{scope}[shift={(11.5,-1.25)}]
    \fill[blue!8] (0,0) rectangle (2.0,1.5);
    \draw (0,0) rectangle (2.0,1.5);
    \fill[blue!14] (2.0,0)--(2.5,0.5)--(2.5,2.0)--(2.0,1.5)--cycle;
    \draw (0,1.5)--(0.5,2.0)--(2.5,2.0)--(2.0,1.5);
    \draw (2.0,0)--(2.5,0.5);
    \foreach \x in {0.4,0.8,1.2,1.6} \draw[gray!45] (\x,0)--(\x,1.5);
    \foreach \y in {0.5,1.0} \draw[gray!45] (0,\y)--(2.0,\y);
    \node[lbl] at (1.0,-0.33) {lag $t\rightarrow$};
    \node[lbl,rotate=90] at (-0.32,0.75) {feature $i$};
    \node[lbl,rotate=45] at (2.72,0.62) {horizon $h$};
    \node[font=\small\bfseries] at (1.0,2.35) {$\Phi$};
  \end{scope}
  \draw[->,thick] (game.east) -- (11.5,-0.5);
  \begin{scope}[shift={(15.0,0.1)}]
    \draw[->] (0,0)--(1.6,0) node[right,font=\scriptsize]{$t$};
    \draw[->] (0,0)--(0,0.8);
    \draw[thick,blue!60!black] plot[smooth]
      coordinates {(0.1,0.1)(0.45,0.28)(0.8,0.7)(1.02,0.8)(1.28,0.46)(1.5,0.2)};
    \draw[dashed,red!70!black] (1.0,0)--(1.0,0.8) node[above,font=\tiny]{$\bar t_i$};
    \node[lbl,anchor=west] at (1.75,0.4) {\textbf{when}: $g_i(t)$};
  \end{scope}
  \begin{scope}[shift={(15.0,-1.15)}]
    \draw[->] (0,0)--(1.6,0) node[right,font=\scriptsize]{$h$};
    \draw[->] (0,0)--(0,0.8);
    \draw[thick,green!45!black] plot[smooth]
      coordinates {(0.1,0.75)(0.45,0.54)(0.8,0.4)(1.15,0.28)(1.5,0.18)};
    \node[lbl,anchor=west] at (1.75,0.4) {\textbf{how much}: $D_i(h)$};
  \end{scope}
  \draw[->,thick,gray!70] (14.1,-0.3) to[out=25,in=180]
    node[pos=0.5,above,font=\tiny]{avg $n,h,k$} (14.95,0.3);
  \draw[->,thick,gray!70] (14.1,-0.85) to[out=-25,in=180]
    node[pos=0.5,below,font=\tiny]{avg $t,n,k$} (14.95,-0.95);
\end{tikzpicture}}
\caption{\textit{TempShap} at a glance. \emph{(left)} Players are the input
\emph{cells} $(t,i)$, not whole feature columns (which mask all $T$ lags
at once and lose \emph{which} lag matters); a per-output-cell loss game
$v(S)=-(Y-\hat Y|_S)^2$ yields \emph{(right)} the tensor
$\Phi\in\mathbb{R}^{N\times H\times K\times T\times M}$, whose profiles
(averages of $|\Phi|$ over the other axes) give $g_i(t)$ (\emph{when}) and
$D_i(h)$ (\emph{how much}).}
\label{fig:schematic}
\end{figure}
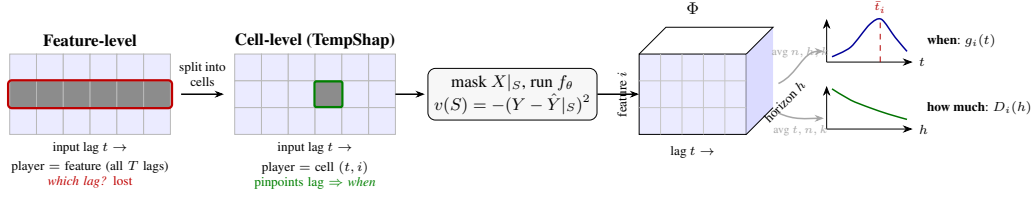

\subsection{Setting and cell-level masking}
\label{sec:method-notation}

Let $X\in\mathbb{R}^{N\times T\times M}$ be $N$ input windows ($T$ lags,
$M$ features), $Y\in\mathbb{R}^{N\times H\times K}$ the targets ($H$ forecast steps, $K$ target channels), and
$\hat Y=f_\theta(X)$ with
$f_\theta:\mathbb{R}^{T\times M}\to\mathbb{R}^{H\times K}$. We index samples
by $n$, input lags by $t$, forecast steps by $h$, target channels by $k$,
and input features by $i$. The players are the $T\cdot M$ input
\emph{cells}: for a coalition
$S\subseteq\mathcal{P}:=\{(t,i)\}$, the masked input $X|_S$ keeps
$X_{n,t,i}$ for $(t,i)\in S$ and sets $b_i$ otherwise, with
$\hat Y|_S:=f_\theta(X|_S)$, where the baseline $b_i$ is the per-feature
\emph{training mean}, the standard interventional/marginal choice
($b_i\approx0$ on standardized inputs). Plugging in the mean approximates the interventional
loss $\mathbb{E}_{X'}\!\left[\,\ell\bigl(Y, f(x_S,
X'_{-S})\bigr)\right]$ \citep{chen2020true,covert2020understanding}. The
two agree on our data, so we use the mean throughout;
\Cref{app:sensitivity} verifies this against a Monte-Carlo
marginal expectation and conditional imputers. Masking a single cell
replaces feature $i$ at lag $t$ alone; this per-cell mask, not one
broadcast over all $T$ lags, is
what makes the input-time axis identifiable.

\subsection{Per-cell characteristic function and Shapley values}
\label{sec:method-charfn}

The characteristic function is the per-output-cell negative squared
error
\begin{equation}
v_{n,h,k}(S)=-\bigl(Y_{n,h,k}-\hat Y_{n,h,k}|_S\bigr)^2,
\label{eq:charfn}
\end{equation}
a \emph{loss-based} game in the sense of \textit{SAGE}
\citep{covert2020understanding}, structured across the $(n,h,k)$ output
axes. Because $v$ uses the realized $Y$, \textit{TempShap} is a
\emph{retrospective error attribution}, not an online
``why-this-prediction'' explanation (a prediction-based variant
$v(S)=\hat Y(X|_S)$ is a robustness check, \Cref{app:detail}). The
Shapley value of cell $(t,i)$ is the standard weighted marginal average
\begin{equation}
\phi_{n,h,k}(t,i)=\!\!\sum_{S\subseteq\mathcal{P}\setminus\{(t,i)\}}\!\!
\frac{|S|!\,(|\mathcal{P}|-|S|-1)!}{|\mathcal{P}|!}
\bigl(v_{n,h,k}(S\cup\{(t,i)\})-v_{n,h,k}(S)\bigr),
\label{eq:shapley}
\end{equation}
giving $\Phi\in\mathbb{R}^{N\times H\times K\times T\times M}$ with
per-output efficiency
$\sum_{(t,i)}\phi_{n,h,k}(t,i)=v_{n,h,k}(\mathcal{P})-v_{n,h,k}(\varnothing)$.
Because $v$ is the negative error, a \emph{positive} $\phi_{n,h,k}(t,i)$
means retaining that cell reduces the error relative to the all-masked
reference, and a \emph{negative} one means it increases the error; this
sign is what the deletion experiment of \Cref{sec:results-efficiency}
exploits. Marginalising our tensor over $t$ does \emph{not} recover the
$M$-feature time-broadcast game, whose coalitions force all $T$ cells of
a feature in or out together; the two coalition spaces differ.

\subsection{Estimators and cost}
\label{sec:method-cost}

Full enumeration requires $2^{TM}$ coalitions ($2^{336}$ for Jena),
so we estimate the cell game by sampling.
Our default estimator, \textit{cell-pShapley}, starts from
the fully masked input and restores cells one at a time in a uniformly
random order. Each step records the change in negative squared error for
every output cell. For brevity, we omit the output indices $(n,h,k)$ and
write $c=(t,i)$ for an input cell. Let $\pi$ order $\mathcal{P}$, with
$\pi(c)\in\{1,\ldots,|\mathcal{P}|\}$ the position of $c$.
The cells restored before $c$ are
$S_\pi(c)=\{c'\in\mathcal{P}:\pi(c')<\pi(c)\}$, so its marginal contribution is
\begin{equation}
\Delta_c(\pi)=v\bigl(S_\pi(c)\cup\{c\}\bigr)-v\bigl(S_\pi(c)\bigr).
\label{eq:marginal}
\end{equation}
We form $m_{\pi}$ \textit{antithetic pairs}, indexed by
$r=1,\ldots,m_{\pi}$. For each pair, we sample a base permutation $\pi_r$
independently and uniformly and traverse both $\pi_r$ and its reverse,
$\bar\pi_r(c)=|\mathcal{P}|+1-\pi_r(c)$, starting from the fully masked input
each time. Averaging their marginal contributions gives
\begin{equation}
\hat\phi_c=\frac{1}{m_{\pi}}\sum_{r=1}^{m_{\pi}}\bar\Delta_c(\pi_r),
\qquad
\bar\Delta_c(\pi_r)=\tfrac12\bigl[\Delta_c(\pi_r)+\Delta_c(\bar\pi_r)\bigr].
\label{eq:estimator}
\end{equation}
For each permutation, we evaluate the fully masked input and then
re-evaluate after restoring each cell. These $TM+1$ evaluations give
one marginal contribution per cell \citep{castro2009polynomial}.
\Cref{sec:method-properties} establishes its
guarantees and analyzes pairing.
\textit{cell-ShapG}
instead enumerates all sub-coalitions of each cell's bounded graph
neighborhood \citep{zhao2025shapg} (built once on the training split
\citep{zhao2025complete}); the restriction is \emph{structural}, so it
targets the graph-\emph{restricted} game and is \emph{biased} for the full
cell game, a comparison baseline only (\Cref{sec:results-competitor}).

\emph{Cost} (model-forwards; one returns the whole $H\times K$ output):
\textit{cell-pShapley} $O(N\,TM\,m_{\pi})$, linear in the budget; \textit{cell-ShapG}
$O(N\,TM\,2^{m_G})$, exponential in the neighborhood cap $m_G$; \textit{WindowSHAP}
$O(N\,2^{C})$ over its $C{=}T/W$ time chunks; \textit{PAX-TS} $O(N\,TM\,S)$ for $S$
perturbation scales (no coalitions); \textit{GS-SHAP} evaluates few group coalitions
but is dominated by its grouping step. Only \textit{cell-pShapley} is
cell-resolved, unbiased, and linear in its budget.

\subsection{Properties of cell-pShapley}
\label{sec:method-properties}

Fix the deterministic model, input, target, and masking baseline. For each output
cell $(n,h,k)$ separately, suppress these indices and draw the base permutations
$\pi_r$ independently and uniformly, pairing each with its reverse.
All expectations and variances below concern only this permutation randomness.

\begin{proposition}[Unbiasedness]
\label{prop:unbiasedness}
For every $m_{\pi}\ge1$, $\mathbb{E}[\hat\phi_c]=\phi_c$.
\end{proposition}

\begin{proof}
For a uniform permutation $\pi$ and any
$S\subseteq\mathcal{P}\setminus\{c\}$, exactly
$|S|!\,(|\mathcal{P}|-|S|-1)!$ permutations place precisely $S$ before $c$.
Consequently, $\Pr[S_\pi(c)=S]$ is the weight of $S$ in
\eqref{eq:shapley}, and $\mathbb{E}[\Delta_c(\pi)]=\phi_c$.
Reversal is a bijection on the permutation space, so $\bar\pi$ is uniform too.
Both $\bar\Delta_c$ and its average over pairs are therefore unbiased.
\end{proof}

\begin{proposition}[Efficiency]
\label{prop:efficiency}
For every $m_{\pi}\ge1$ and every collection of sampled permutation pairs,
$\sum_c\hat\phi_c=v(\mathcal{P})-v(\varnothing)$.
\end{proposition}

\begin{proof}
Order the cells as $c_1,\ldots,c_{|\mathcal{P}|}$ along any permutation $\pi$.
Averaging the following telescoping identity over the sampled forward and
reverse permutations proves the claim:
\begin{equation}
\sum_{j=1}^{|\mathcal{P}|}\Delta_{c_j}(\pi)
=\sum_{j=1}^{|\mathcal{P}|}
\left[v(\{c_1,\ldots,c_j\})-v(\{c_1,\ldots,c_{j-1}\})\right]
=v(\mathcal{P})-v(\varnothing).\qedhere
\end{equation}
\end{proof}

This property holds without an additional constraint; other estimators, such as
constrained \textit{KernelSHAP} \citep{covert2021kernelshap}, can also enforce
finite-sample efficiency.

\begin{proposition}[Variance of the paired estimator]
\label{prop:variance}
Let
$\sigma_c^2=\operatorname{Var}_\pi(\bar\Delta_c(\pi))$. Then
\begin{equation}
\operatorname{Var}(\hat\phi_c)=\frac{\sigma_c^2}{m_{\pi}}
=\frac{\operatorname{Var}(\Delta_c(\pi))+
\operatorname{Cov}(\Delta_c(\pi),\Delta_c(\bar\pi))}{2m_{\pi}}.
\label{eq:paired-variance}
\end{equation}
Here $\sigma_c^2$ depends only on interactions of order three and above.
\end{proposition}

\begin{proof}
Function $v(S)$ can be represented as follows:
$v(S)=\sum_{A\subseteq S}a_A$, where
$a_A=\sum_{B\subseteq A}(-1)^{|A|-|B|}v(B)$,
and \eqref{eq:marginal} takes the form
\begin{equation}
\Delta_c(\pi)=\sum_{A\ni c}a_A I_{A,c}(\pi),
\qquad I_{A,c}(\pi)=\mathbf{1}\{A\setminus\{c\}\subseteq S_\pi(c)\}.
\label{eq:marginal-mobius}
\end{equation}
A singleton contributes the constant $a_{\{c\}}$.
For $A=\{c,j\}$, exactly one of $\pi$ and $\bar\pi$ puts $j$ before $c$,
so $[I_{A,c}(\pi)+I_{A,c}(\bar\pi)]/2=1/2$ for every permutation $\pi$.
For larger $A$, define
$B_{A,c}(\pi)=[I_{A,c}(\pi)+I_{A,c}(\bar\pi)]/2$.
The cell $c$ is last in $A$ with probability $1/|A|$ in either uniform
ordering, so $\mathbb{E}[B_{A,c}]=1/|A|$. We obtain
\begin{equation}
\bar\Delta_c(\pi)
=a_{\{c\}}+\frac12\sum_{j\ne c}a_{\{c,j\}}
+\sum_{\substack{A\ni c\\|A|\ge3}}a_A B_{A,c}(\pi).
\end{equation}
Only the last sum is random, and its variance is
\begin{equation}
\sigma_c^2
=\sum_{\substack{A\ni c,\ |A|\ge3\\D\ni c,\ |D|\ge3}}
 a_A a_D\operatorname{Cov}_\pi(B_{A,c},B_{D,c}).
\label{eq:interaction-variance}
\end{equation}
This expression makes the dependence on interaction strengths,
signs, and the joint ordering of their cells explicit. Independence across the
$m_{\pi}$ pairs gives $\operatorname{Var}(\hat\phi_c)=\sigma_c^2/m_{\pi}$.
Since the forward and reverse marginals have the same variance, expanding the
variance of their average gives \eqref{eq:paired-variance}. Unbiasedness then
makes RMSE equal to $\sigma_c/\sqrt{m_{\pi}}$.
\end{proof}

At equal cost, $2m_{\pi}$ independent permutations have the variance in
\eqref{eq:paired-variance} without the covariance term, so pairing helps when
that covariance is negative.
Pairing can also reduce variance due to higher-order interactions.
For example, in the pure three-player interaction game
$v(S)=\mathbf{1}\{\{1,2,3\}\subseteq S\}$, a cell's unpaired marginal is one
with probability $1/3$ and zero otherwise. Averaging two independent permutations
has variance $1/9$. Under reversal, its paired marginal is $1/2$ when that cell
is first or last and zero when it is in the middle, giving variance $1/18$.
Thus a variance reduction does not imply the presence of second-order terms.

For games with only single-cell and pairwise effects, one
permutation and its reverse are enough to obtain exact Shapley values
\citep[Proposition~3.1]{mayer2025paired}.
The following corollary applies this result to our squared-loss game.

\begin{corollary}[Exactness for affine forecasters]
\label{cor:affine-exactness}
For an affine forecaster under fixed-baseline masking and squared loss,
averaging the marginal contributions from any permutation and its reverse
gives the exact cell Shapley values.
\end{corollary}

\begin{proof}
Fix an output cell with target $y$. Let $\hat y|_S$ be its prediction
when input cells in $S$ are retained and all others are masked.
In particular, $S=\varnothing$ masks all cells, while $S=\{c\}$ retains
only cell $c$. Thus $q_c=\hat y|_{\{c\}}-\hat y|_{\varnothing}$ is the
change in prediction from restoring cell $c$ alone.

For an affine forecaster, this change does not depend on which other
cells are retained. The changes therefore add up:
$\hat y|_S=\hat y|_{\varnothing}+\sum_{c\in S}q_c$.
Write $e_0=y-\hat y|_{\varnothing}$ for the error with all cells masked.
Expanding the squared loss gives
\begin{equation}
\begin{aligned}
v(S)&=-(y-\hat y|_S)^2=-\left(e_0-\sum_{c\in S}q_c\right)^2\\
&=-e_0^2+2e_0\sum_{c\in S}q_c-\sum_{c\in S}q_c^2
  -2\sum_{\substack{c,j\in S\\c<j}}q_cq_j\\
&=-e_0^2+\sum_{c\in S}(2e_0q_c-q_c^2)
  -2\sum_{\substack{c,j\in S\\c<j}}q_cq_j.
\end{aligned}
\label{eq:linear-loss-expansion}
\end{equation}
The constant term does not affect marginal contributions. Every remaining
term involves either one cell or a pair of cells.
\Eqref{eq:interaction-variance} gives $\sigma_c^2=0$. Since the paired estimator
is unbiased, even a single permutation pair yields the exact Shapley values.
\end{proof}

DLinear \citep{zeng2023transformers} adds linear forecasts of a moving-average trend and
its residual, so it is an affine forecaster covered by \Cref{cor:affine-exactness}.
\Cref{tab:budget} tests this prediction for the squared-loss game, comparing
$m_{\pi}$ antithetic pairs with $2m_{\pi}$ independent permutations at equal cost.
With pairing, DLinear agrees with the high-budget reference up to numerical
precision even at $m_{\pi}=1$; independent permutations retain sampling error.
For nonlinear forecasters, pairing gains vary by dataset and model and are
small for the Exchange-Rate Transformer. These gains reflect the covariance
in \eqref{eq:paired-variance}, rather than the fraction of second-order
interactions. The finite-reference error is discussed in \Cref{app:sensitivity}.

\subsection{Temporal profiles and rank settling}
\label{sec:method-diag}

Write $|\Phi|$ for the cell-wise absolute attribution. The two profiles
below are averages of $|\Phi|$ over the axes not shown.

The \textit{temporal-localization profile}
$g_i(t)=\mathbb{E}_{n,h,k}\,|\phi_{n,h,k}(t,i)|$ describes feature $i$'s
input-lag importance; its centroid $\bar t_i=\sum_t t\,g_i(t)/\sum_t g_i(t)$
answers \emph{when}, and $1-\mathcal{H}[g_i]/\log T$ how localized
($\mathcal{H}$ = entropy over $t$).

The \textit{horizon-impact profile}
$D_i(h)=\mathbb{E}_{t,n,k}\,|\phi_{n,h,k}(t,i)|$ measures the magnitude of
feature $i$'s contribution at horizon step $h$; its trend answers
\emph{how much} the impact changes across the forecast (it need not be
monotone; ``decay'' is reserved for the planted $\gamma^h$ case).

For \textit{rank settling}, let $\pi^{\le h}$ be the descending feature
ranking by cumulative magnitude up to $h$. The correlation $\tau_h=\mathrm{KendallTau}(\pi^{\le h},\pi^{\le H})$
and the settling time
$h_{\mathrm{settle}}(\theta)=\min\{h:\tau_s\ge\theta\ \forall s\ge h\}$
measure how fast the \emph{ranking} stabilizes. This is distinct from
$D_i(h)$, which measures \emph{magnitude}: a ranking can be stable from
$h=1$ while every magnitude decays (\Cref{sec:results-synth}, C3).

\section{Experimental protocol}
\label{sec:protocol}

We use Jena climate ($M=14$, hourly), ETTh1 \citep{zhou2021informer} ($M=7$,
hourly), and LSTNet Exchange-Rate \citep{lai2018modeling} ($M=8$, daily).
Input and forecast lengths are $T=H=24$, with unit-shift windows and
Informer 60/20/20 splits (Jena 70/20/10). We report
single-target ($K=1$) as primary, since it isolates the temporal axes; the
targets are temperature (Jena), oil temperature OT (ETTh1), and the AUS
rate (Exchange-Rate).
On each dataset we explain three \emph{distinct} architecture families,
trained with a \emph{single unified budget} (40 epochs, early-stopping
patience 6, Adam $10^{-3}$): an \textit{LSTM} (recurrent), a small
\textit{Transformer} encoder (attention, learned positional encoding), and
\textit{DLinear} (linear, series-decomposition).

\textit{TempShap} is estimator-agnostic. We evaluate \textit{cell-pShapley}
and \textit{cell-ShapG}, the two estimators in \Cref{sec:method-cost},
with \textit{TimeSHAP-c} and \textit{cell-Exact} as references.
\textit{cell-pShapley} is the deployed default ($m_{\pi}{=}32$ for
profiles, fewer for rankings) and \textit{cell-ShapG} ($m_G{=}5$) the
graph baseline. The references are \textit{TimeSHAP-c}, a loss-based
cell-level attribution on the \textit{TimeSHAP} player set \citep{bento2021timeshap}
solved by leverage-score sampling \citep{musco2025leverage}, used as an
independent agreement check, and \textit{cell-Exact}, an oracle for a
tractable small case.

For deletion faithfulness, we adapt the weighted-slope score of
\citet{zhao2025shapg} to the cell level: cells are deleted (masked to the
mean) in a method's descending importance order, in $10\%$ steps of the
$T\cdot M$ grid. Let $S_j$ contain the cells retained after step $j$ and
$X_j=X|_{S_j}$. Using the payoff in \eqref{eq:charfn}, define
$\mathrm{Perf}(X_j)=\mathbb{E}_{n,h,k}[v_{n,h,k}(S_j)]$, the negative squared
error averaged over evaluation windows and forecast outputs. Then
\begin{equation}
U_j=\frac{\mathrm{Perf}(X_j)-\mathrm{Perf}(X_\varnothing)}
{\mathrm{Perf}(X_{\mathcal{P}})-\mathrm{Perf}(X_\varnothing)},
\qquad
S=\sum_{j}\beta^{\,j-1}(U_{j-1}-U_j)\ \ (\beta=0.8),
\end{equation}
so $U{=}1$ denotes the full input and $U{=}0$ the all-masked input.
Larger $S$ indicates earlier performance drops and greater deletion faithfulness.
We do not clip
$U_j$, so an input worse than the all-masked reference gives $U_j<0$ and
$S>1$, flagging harmful surviving inputs. Every method is scored under the
\emph{identical} mean-mask evaluator, so $S$ compares rankings, not
maskers; rankings are estimated on $24$ test windows and $S$ on the
following disjoint $24$,\footnote{Here, ``disjoint'' means
that the ranking and evaluation sets contain different window instances, sampled
without replacement from the same shuffled test split. It does not exclude
overlap in their underlying input or target timestamps, which affects $1$--$3\%$
of window pairs. A strictly timestamp-disjoint control leaves the ranking
comparison of \Cref{sec:results-efficiency} unchanged.} repeated across
model-training seeds. Besides
\textit{cell-pShapley} we rank cells by \textit{PAX-TS} sensitivity, by
single-cell leave-one-out (\textit{cell-LOO},
$v(\mathcal{P})-v(\mathcal{P}\setminus\{(t,i)\})$), by \textit{cell-ShapG} and \textit{TimeSHAP-c}, and by the unsigned
\textit{cell-pShapley} ranking, which averages $|\phi|$ over ranking windows
and forecast outputs.

\section{Results}
\label{sec:results}

\subsection{Ground-truth validation}
\label{sec:results-synth}

\input{tables/synth_c1.tex}

\begin{figure}[!htb]
\centering
\includegraphics[width=\linewidth]{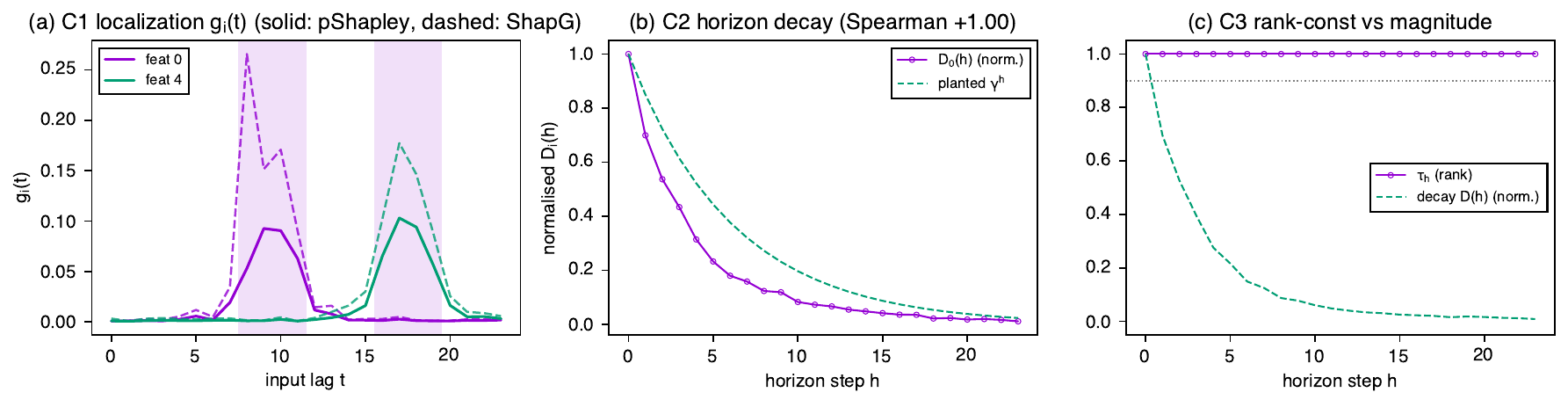}
\caption{Synthetic ground-truth validation. (a) C1: localization
profiles $g_i(t)$ peak at the planted active lags (shaded). (b) C2: the
horizon-impact profile $D_0(h)$ tracks the planted $\gamma^h$. (c) C3: the
ranking settles immediately ($\tau_h\approx1$) while the magnitude
$D(h)$ decays; the ranking and the magnitude are distinct.}
\label{fig:synth-when}
\end{figure}

We validate these profiles against three planted \emph{constructions}
(C1--C3; details in \Cref{app:detail}, summarized in
\Cref{fig:synth-when}).
\textit{C1 (localization).} With feature~0 predictive only at lags
$[8,12)$ and feature~4 only at $[16,20)$, all three cell estimators
recover both planted lags to within a lag and rank both signal features
top-2 (\Cref{tab:synth-c1}). \textit{C2
(decay)} plants a $\gamma^h$-decaying dependence ($\gamma=0.85$), and
\textit{C3 (rank-constant, magnitude-decaying)} scales four fixed-ranking
features by $\gamma^h$; the estimated profiles reproduce both patterns
(\Cref{fig:synth-when}b,c).

In the small exact-oracle experiment ($T=6$, $M=2$), the target is the
mean of feature~0 at lags 2 and 3 plus noise, giving a planted centroid of 2.5.
On this tractable case ($2^{12}$ coalitions),
\textit{cell-pShapley} matches the exact cell oracle's centroid and \textit{cell-ShapG} is
close (\Cref{tab:synth-c1}); the non-signal cells carry near-zero
attribution, so the oracle is read mainly through the centroid.

For the real-data analytic oracle, we use DLinear's prediction as the
payoff, $v(S)=\hat y|_S$, for a fixed output cell.
Let $J_{t,i}=\partial\hat y/\partial X_{t,i}$.
Since DLinear is affine, restoring cell $(t,i)$ from its baseline
always changes the prediction by the same amount, giving the exact
Shapley value $\phi_{t,i}=J_{t,i}(X_{t,i}-b_i)$.
\textit{cell-pShapley} recovers these values on all three real DLinear forecasters
(signed Pearson $1.00$, per-cell Kendall-$\tau$ $1.00$, $0$ lag centroid
error; \Cref{tab:linear-gt}); \textit{cell-ShapG} is close (Pearson $0.87$--$0.99$,
$\le 0.5$ lags). Here finite-budget exactness follows from constant
prediction marginals; for the quadratic squared-loss game, it follows
from antithetic pairing (\Cref{sec:method-properties}). This prediction-based
oracle validates the estimator; the planted experiments above validate
the loss-based headline game.

\subsection{Competitor comparison}
\label{sec:results-competitor}
Do existing time-series Shapley variants already answer \emph{when}?
We use a variant of C1: the target is the sum of feature~0's mean over
lags $[8,11)$ and feature~4's mean over lags $[16,19)$, plus noise.
\Cref{tab:competitor} compares the methods, including \textit{GroupSegment-SHAP} \citep{kim2026gsshap} and \textit{PAX-TS}
\citep{kreuzer2025paxts} at the authors' settings. Feature- and
chunk-level player sets \citep{nayebi2023windowshap} cannot separate the
two planted lags; only cell-level attribution recovers both, consistently
across three Shapley estimators and a per-cell permutation control
\citep{breiman2001random}. The per-feature profiles show the same picture
(\Cref{fig:competitor}).
Among cell-level methods, only \textit{cell-pShapley} is Shapley-axiomatic
\emph{and} runs at perturbation cost. \textit{PAX-TS}'s six-scale perturbation set
costs the same as $m_{\pi}{=}3$ antithetic pairs, and at that budget
\textit{cell-pShapley} localizes at least as sharply. \textit{cell-ShapG} needs $6.4\times$
more evaluations and carries the structural offset of
\Cref{sec:method-cost} (\Cref{fig:real-profiles}). \textit{GS-SHAP} is
the slowest despite the fewest coalitions: its HSIC (Hilbert--Schmidt
independence criterion) and MMD (maximum mean discrepancy) kernel
grouping, not the game, dominates its wall-clock.

\input{tables/competitor.tex}

\subsection{Redundancy and deletion faithfulness}
\label{sec:results-efficiency}

We next test whether redundancy accounting and signed loss
contributions yield stronger deletion rankings.

In the redundancy counterexample, features $0,1$ are equally-noisy copies of one latent signal driving the
target (six further features are noise); we ask each method what share of
the pair's \emph{joint} value
$\Delta_{\mathrm{pair}}=v(\mathcal{P})-v(\mathcal{P}\setminus\{0,1\})$ it
credits (\Cref{tab:redundancy}). Leave-one-out credits about half:
removing either copy is covered by the survivor, the classic blind spot of
one-at-a-time scores \mbox{\citep{covert2020understanding}}. In the high-correlation cases, \textit{cell-pShapley} credits
the pair approximately its joint value, while efficiency constrains the total
attribution across all cells. Once the redundancy is gone (planted correlation $0.20$) the two
accountings are indistinguishable, so the gap is redundancy-specific.
The same construction scored with feature-level Exact Shapley behaves
differently again (\Cref{app:detail}).

To evaluate deletion faithfulness, we delete cells in each method's importance order and record the
performance drop (\Cref{sec:protocol}, \Cref{tab:cellkpi},
\Cref{fig:cellkpi}). Under mean-mask deletion, the
signed \textit{cell-pShapley} ranking beats \textit{cell-LOO} on
every dataset-model case and \textit{PAX-TS} on $8$ of $9$; over the $27$ paired
runs, exact sign tests give $p<10^{-6}$ and $p<10^{-4}$.
\textit{cell-ShapG} and \textit{TimeSHAP-c} also beat \textit{PAX-TS} in
$24/27$ runs each, so
within this evaluator the advantage belongs to the
Shapley structure rather than to one estimator.
Deleting successive batches of ranked cells tests whether a ranking
leaves substitutes for deleted cells in place. The sign
matters too: the unsigned $|\phi|$ variant falls back to \textit{PAX-TS}
($14/27$). A signed ranking deletes the loss-reducing cells first, so the
loss-increasing ones survive and actively mislead the model midway
through the curve (\Cref{fig:cellkpi}). Where redundancy is rare, as on
ETTh1 with its dominant OT feature, all methods nearly tie.
This advantage does not persist under a
conditional-imputation evaluator (\Cref{app:onmanifold}).

\input{tables/redundancy.tex}

\input{tables/cellkpi.tex}

\begin{figure}[!htb]
\centering
\includegraphics[width=\linewidth]{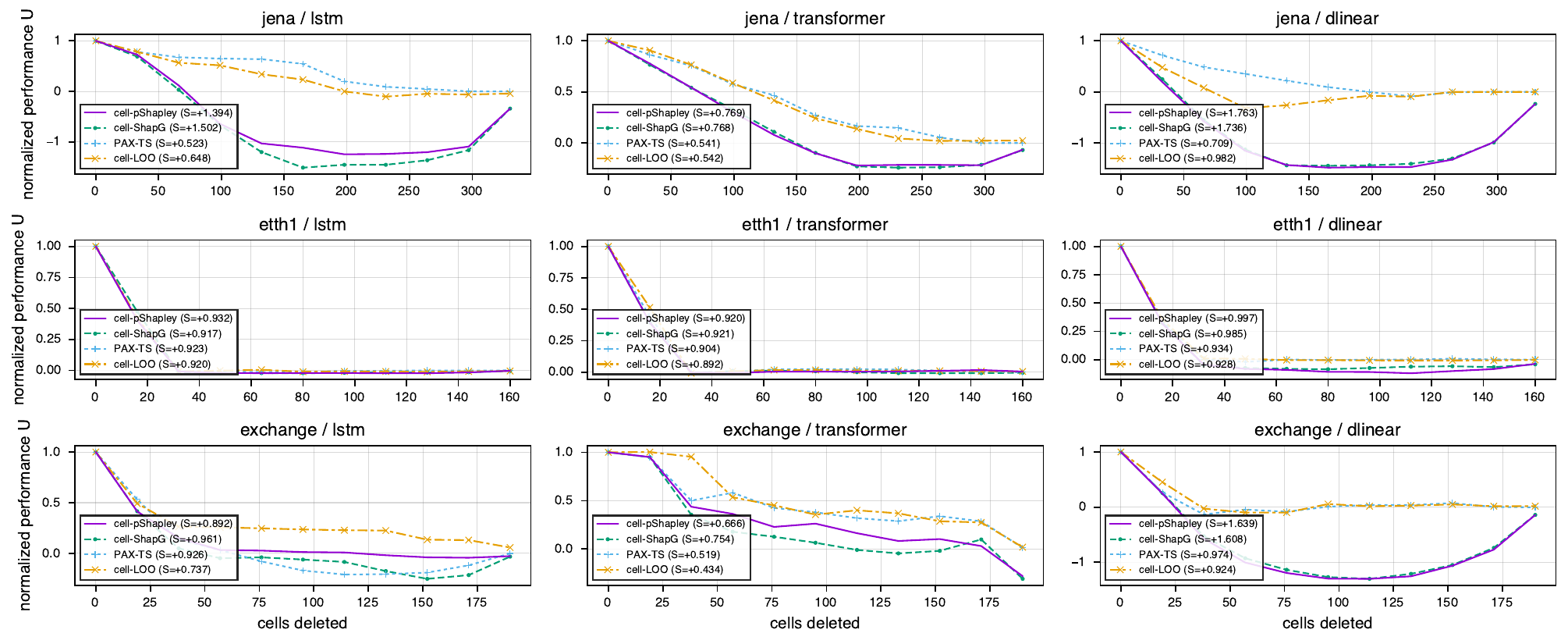}
\caption{Cell-level deletion curves for one of the three training seeds
averaged in \Cref{tab:cellkpi} (held-out windows; $U{=}1$ full model,
$U{=}0$ all-masked; legend gives $S$; \textit{cell-pShapley} at $m_{\pi}{=}16$). Deleting
in \textit{cell-pShapley} order drives $U$ below the all-masked baseline on the
redundant Jena panels while \textit{PAX-TS} and \textit{cell-LOO} plateau; the methods
nearly tie on ETTh1.}
\label{fig:cellkpi}
\end{figure}

\subsection{Real data}
\label{sec:results-real}
The profiles resolve the $(t,i,h)$ structure on every real forecaster
(\Cref{tab:real}; per-feature profiles in
\Cref{fig:real-profiles}). A single feature dominates across
families on two of the three datasets (OT on ETTh1, AUS on
Exchange-Rate).

There is no consistent recency ordering across architectures:
the most-recent family changes across datasets (\Cref{tab:real}). Most
$g_i(t)$ profiles are near uniform; their departures are dataset- and
model-specific, revealing structure that a single centroid cannot resolve.

The rank-settling $\tau_h$ distinguishes stable from evolving rankings
and separates cases that no scalar attribution distinguishes: some
rankings are fixed from $h=1$, others keep reordering well into the
horizon (\Cref{tab:real}, \Cref{fig:conv}).

\input{tables/real.tex}

Baseline, payoff and budget sensitivity controls are reported
in \Cref{app:sensitivity}.
At the reported sampling budget, the tested profiles show
little variation across independent permutation draws in eight of nine
dataset--model cases, with the Exchange-Rate Transformer as the exception
(\Cref{fig:profile-stability}).

\section{Conclusion}
\label{sec:conclusion}
\textit{TempShap} resolves input lag and forecast horizon through a
cell-level loss game, validated against planted ground truth and analytic
oracles. Permutation sharing makes estimation practical at localized
perturbation cost. Its signed Shapley rankings improve held-out deletion
faithfulness under mean masking, with gains reproduced
across estimators, supporting applications such as
pruning and masking.

There are some limitations.
\textit{TempShap} explains error after the fact, not a prediction online. The
redundancy mechanism is shown causally only on the synthetic
counterexample. Shapley rankings are not generally optimal for greedy deletion; the observed
advantage is empirical. We evaluate one model per family rather than multiple
tuned instances.

Three directions for future work can be proposed: (a) Explaining
this dependence on the deletion protocol and developing rankings suited to
conditional deletion remain open problems. (b) Multi-target ($K>1$) extensions must account for each target's own-history
contribution. (c) Larger $T\cdot M$ and further architectures (e.g.\ PatchTST,
TFT) would test scalability.

\bibliographystyle{icomp2026_conference}
\bibliography{icomp2026_conference}

\clearpage
\appendix
\section{Construction details and ablations}
\label{sec:appendix}
\label{app:detail}

\subsection{Setup and synthetic constructions}
\label{app:setup}

All timings are measured on one machine: an
NVIDIA RTX 4090 (24\,GB), a 32-thread Intel Core i9-13900F, and
$32$\,GB of RAM.

The synthetic constructions use $T=H=24$, $M=8$ in two correlated blocks (a shared latent per
block at each timestep, i.i.d.\ over $t$ so the planted lag is the unique
information location), $12\,000/2\,000/128$ train/val/test, one LSTM (32
units, 30 epochs). C1 plants feature~0 at lags $[8,12)$ and feature~4 at
$[16,20)$; C2 a single feature at $[8,12)$ with target $\propto\gamma^h$,
$\gamma=0.85$; C3 four features (coefficients $1,0.7,0.5,0.3$) at
$[8,12)$ all scaled by $\gamma^h$. The competitor comparison
(\Cref{tab:competitor}) reuses this construction with narrower bands.

\begin{figure}[!htb]
\centering
\includegraphics[width=\linewidth]{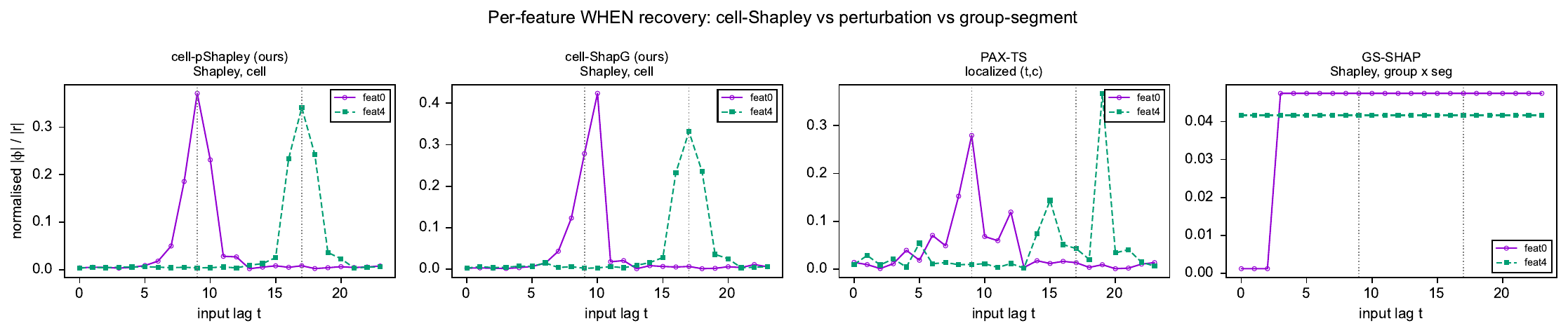}
\caption{Per-feature \emph{when} recovery on the planted band (dotted =
planted lags $9,17$). The cell estimators peak sharply at both lags;
PAX-TS recovers them with a more multi-peaked feat4 profile; GS-SHAP is
flat, since its groups cannot separate the features.}
\label{fig:competitor}
\end{figure}

\begin{figure}[!htb]
\centering
\includegraphics[width=\linewidth]{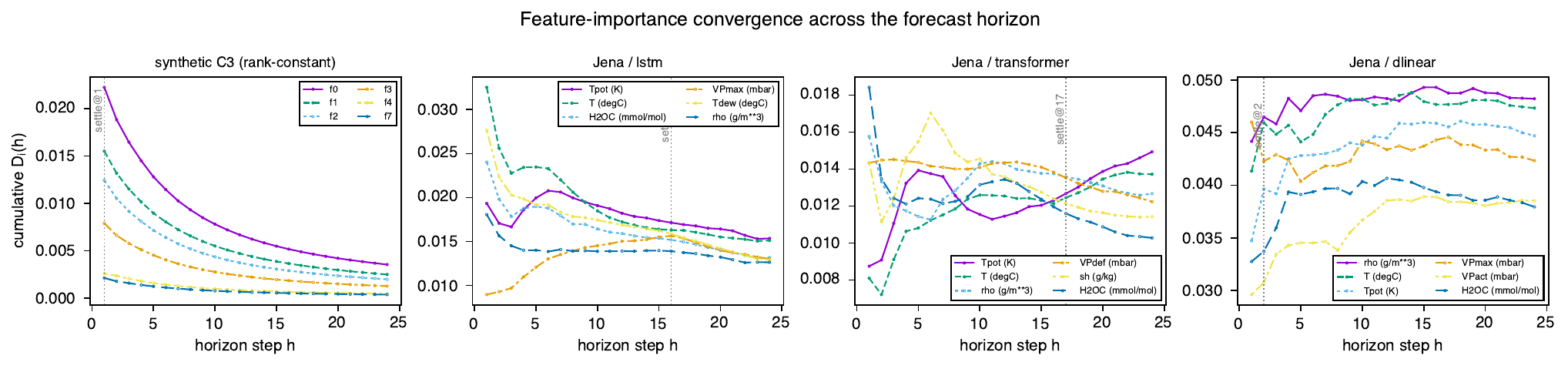}
\caption{Cumulative per-feature importance up to step $h$; the order in
which the curves settle is what $\tau_h$ measures. Synthetic C3 settles at
$h=1$ while its magnitudes decay; on Jena the settling horizon differs by
architecture.}
\label{fig:conv}
\end{figure}

\begin{figure}[!htb]
\centering
\includegraphics[width=\linewidth]{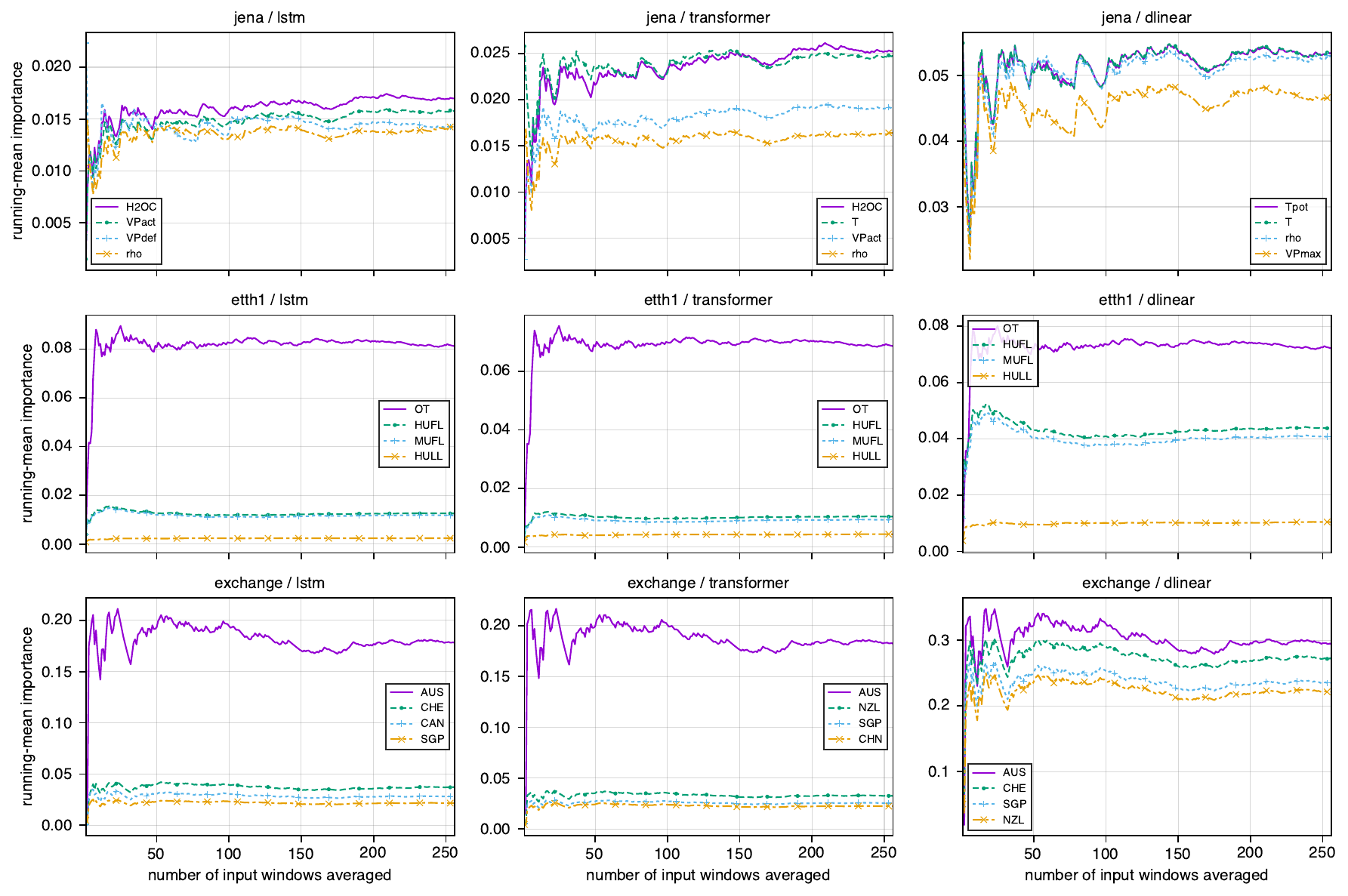}
\caption{Running-mean importance of the top-$4$ features against the
number of input windows averaged (cell-pShapley, real data, $N{=}256$).
The top-$4$ set settles within a few dozen windows, well before the
individual magnitudes stop drifting.}
\label{fig:sampleconv}
\end{figure}

\subsection{Real-data profiles and linear oracle}
\label{app:real-profiles}

\input{tables/linear_gt.tex}

\begin{figure}[!htb]
\centering
\includegraphics[width=\linewidth]{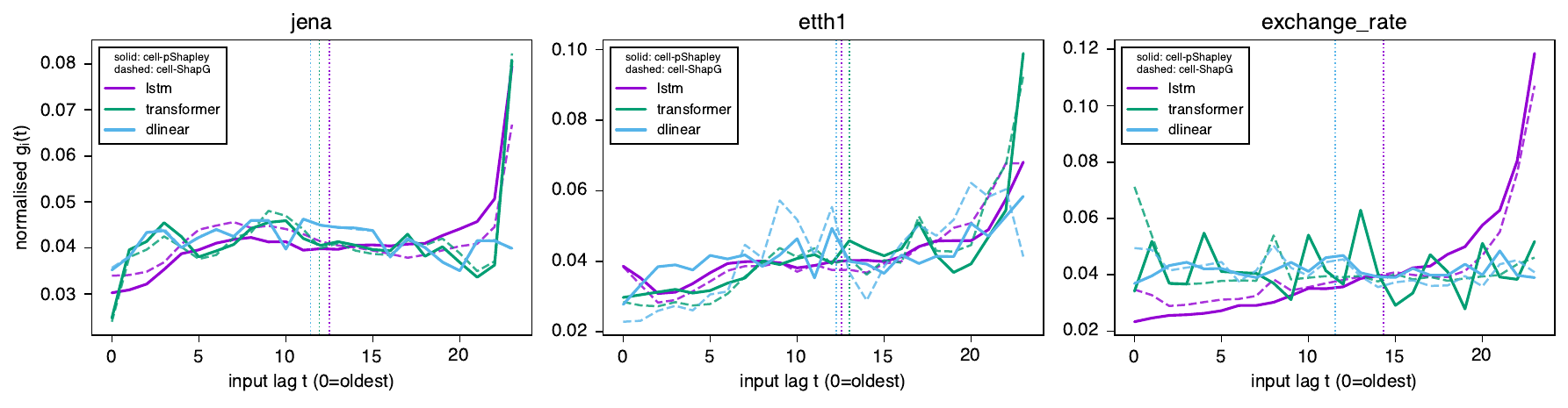}\\[0.6em]
\includegraphics[width=\linewidth]{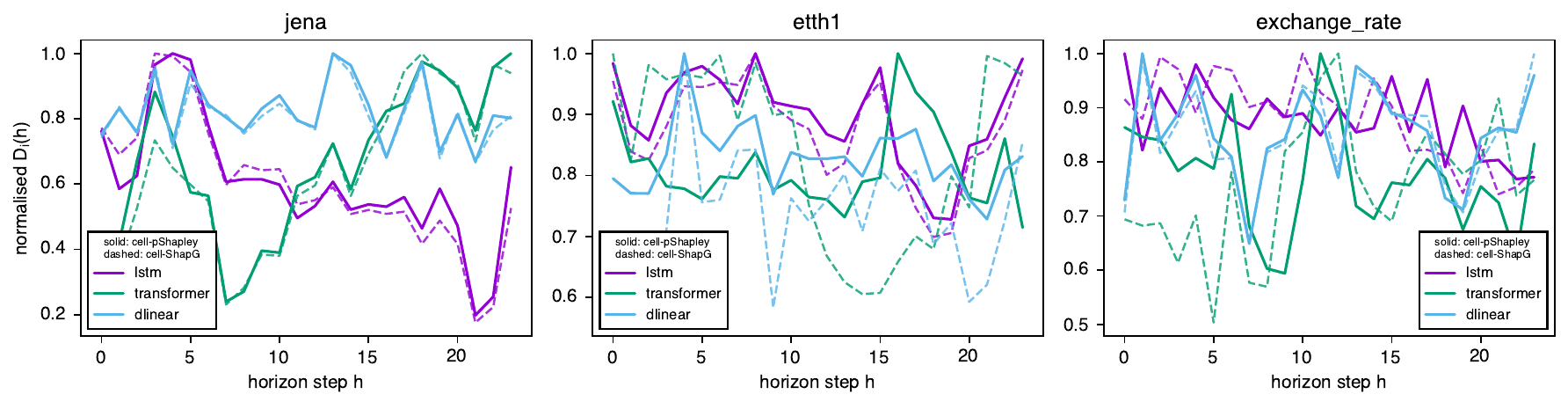}
\caption{Real-data profiles per architecture family, one panel per
dataset. \emph{(top)} \emph{when}: temporal influence $g(t)$, with solid
curves for cell-pShapley, dashed for cell-ShapG, and vertical dotted lines
at the centroid of each cell-pShapley profile. \emph{(bottom)}
\emph{how much}: horizon impact $D(h)$ of the top feature, which is not
monotone, so a feature's influence fluctuates across the horizon rather
than decaying. Both are normalized.}
\label{fig:real-profiles}
\end{figure}

\subsection{Feature-level Exact control}
\label{app:featlevel}

We repeated the redundancy experiment with a different player set: the $M$
whole features rather than the $T\cdot M$ cells, which makes the game
small enough to solve exactly ($2^8$ coalitions). This game gives the
redundant pair $347\pm124\%$ of its joint value, far more than the cell
game does, and splits that credit unevenly between the two copies
($1.6\pm1.3$) from one seed to the next. Two things differ: the coalitions
are no longer the same (\S\ref{sec:method-charfn}), and removing a
feature's entire history at once pushes the input far off the data
manifold when the copies are highly correlated. Efficiency fixes only the
grand total, never how it is divided among players, so how much of it a
group of players captures depends on which players the game has.

\subsection{Sensitivity controls}
\label{app:sensitivity}

\input{tables/sensitivity.tex}

\textit{Reference.} Replacing the mean with a Monte-Carlo marginal
expectation moves the when-centroid by at most $1.4$ lags (median $0.3$;
\Cref{tab:sensitivity}), so the mean plug-in tracks it closely. An
observational game (per-feature temporal Gaussian conditional imputer,
shapr-style \citep{jullum2025shapr,chen2020true}) agrees within $\sim\!1$
lag on $8/9$ cells; the exception, Exchange-Rate LSTM at $\sim\!3$ lags,
is the least-predictable dataset.

\textit{Payoff.} Swapping the error-reduction $v$ (\eqref{eq:charfn}) for
the prediction-based $v(S)=\hat Y|_S$ moves the centroid by $\le 0.3$
lags, and a global $R^2$ game by $\le 1.2$ lags. Which feature ranks first
is unchanged everywhere except in one cell under the prediction-based
game (\Cref{tab:sensitivity}). The profiles do not hinge on the choice of
$v$.

\textit{Budget convergence.} All reported profiles and
rankings are budget-converged at the $m_{\pi}$ used. We swept $m_{\pi}$ from $1$ to $64$ on all
$9$ cases against a reference computed at $128$ pairs
(\Cref{tab:budget}). Cell
values need the budget: they are still improving at $m_{\pi}{=}64$. The
when-centroid does not. On $8/9$ cases it is already within a quarter
of a lag at $m_{\pi}{=}1$. The Exchange-Rate Transformer is the
exception, slow on both. \Cref{fig:sampleconv} shows the same pattern
along the sample axis.

\textit{Sampling stability of the profiles.}
For each of the nine cases, we fix the model, explanation
windows and baseline, and repeat estimation with eight independent permutation
seeds. We assess the feature with the largest attribution mass in an independent
$m_{\pi}{=}128$ reference. At $m_{\pi}{=}32$, mean total-variation distance from
this reference is below $0.021$ for normalised $g_i(t)$ in eight cases, with mean
unnormalised mass differing by less than $1\%$; mean distance for normalised
$D_i(h)$ is at most $0.010$ in all nine. The three DLinear cases are exact at
every budget, as \Cref{cor:affine-exactness} requires. The
Exchange-Rate Transformer's $g_i(t)$ distance is $0.164$, so its fine temporal
structure remains uncertain
(\Cref{fig:profile-stability}).

\begin{figure}[!htb]
\centering
\includegraphics[width=\linewidth]{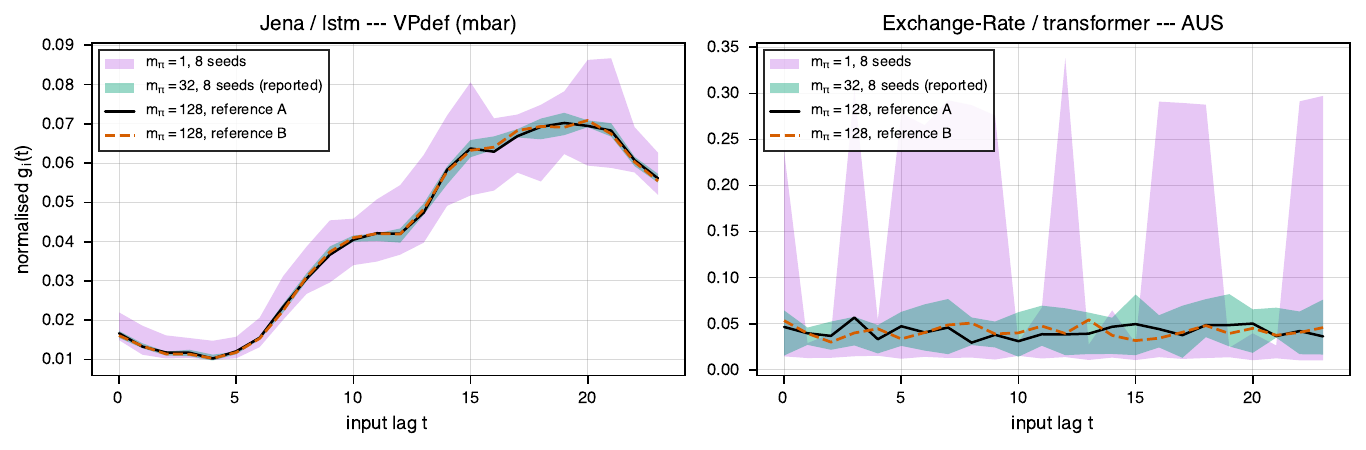}
\caption{Normalised top-feature profiles for two of the nine
cases: Jena LSTM and the Exchange-Rate Transformer. Coloured bands show the
pointwise min--max range across eight permutation seeds at $m_{\pi}{=}1$ and $32$.
The black solid and orange dashed lines show two independent $m_{\pi}{=}128$
references; distances are computed against the black curve.}
\label{fig:profile-stability}
\end{figure}
\textit{Finite-reference error.}
The paired/unpaired comparison in \Cref{tab:budget} is discussed in
\Cref{sec:method-properties}. The observed error decrease is consistent with
Monte-Carlo convergence, but the reference itself uses finitely many pairs.
For a scalar attribution estimate $\hat\phi_{c,m}$ and an independently sampled
reference $\hat\phi_{c,M_{\mathrm{ref}}}$ for the same game,
\begin{equation}
\mathbb{E}\!\left[
(\hat\phi_{c,m}-\hat\phi_{c,M_{\mathrm{ref}}})^2\right]
=\sigma_c^2\left(\frac{1}{m}+\frac{1}{M_{\mathrm{ref}}}\right),
\qquad M_{\mathrm{ref}}=128.
\label{eq:reference-error}
\end{equation}
The expectation includes both independent estimates. Thus error
against this reference need not follow a pure $m^{-1/2}$ rate, particularly as
$m$ approaches the reference budget.

\input{tables/budget.tex}

\subsection{On-manifold deletion}
\label{app:onmanifold}
Keeping rankings, windows and the deletion schedule
fixed, we replace mean masking with the conditional mean of a joint Gaussian
fitted to training windows (\Cref{tab:onmanifold}). Under this evaluator,
cell-pShapley beats PAX-TS on $3/27$ runs and cell-LOO on $11/27$; the unsigned
$|\phi|$ ranking achieves the highest median $S$.

Reconstruction checks (seed $2333$) give $R^2<0$
relative to per-cell training means on $22/72$ tested deletion sets for ETTh1,
versus $0/72$ each for Jena and Exchange-Rate. Excluding all nine ETTh1 runs
leaves cell-pShapley beating cell-LOO on $5/18$ runs and $|\phi|$ beating
PAX-TS on $16/18$.

\input{tables/onmanifold.tex}

\end{document}

%% file: math_commands.tex
\usepackage{amsmath,amsfonts,bm}

\def\eqref#1{equation~\ref{#1}}
\def\Eqref#1{Equation~\ref{#1}}

\def\1{\bm{1}}

\def\ry{{\textnormal{y}}}

\DeclareMathAlphabet{\mathsfit}{\encodingdefault}{\sfdefault}{m}{sl}
\SetMathAlphabet{\mathsfit}{bold}{\encodingdefault}{\sfdefault}{bx}{n}

%% file: tables/synth_c1.tex
\begin{table}[!htb]
\caption{Recovery of the planted signal. Centroids (in lags):
$(\bar t_0,\bar t_4)$ for C1; $\bar t_0$ for Cell-Exact.}
\label{tab:synth-c1}
\centering\small
\setlength{\tabcolsep}{7pt}
\begin{tabular}{llrl}
\toprule
Construction & Estimator & \emph{when} centroid & agreement \\
\midrule
C1 ($T{=}24$, $M{=}8$) & planted & $(9.5, 17.5)$ & --- \\
  & \textbf{cell-pShapley}   & $(9.7, 17.0)$ & both top-2 \\
  & cell-ShapG               & $(9.3, 16.8)$ & both top-2 \\
  & TimeSHAP-c               & $(9.7, 16.9)$ & both top-2 \\
\addlinespace
Cell-Exact ($T{=}6$, $M{=}2$) & exact oracle & $2.54$ & --- \\
  & \textbf{cell-pShapley}   & $2.54$ & $\tau=0.91$ \\
  & cell-ShapG               & $2.46$ & $\tau=0.55$ \\
\bottomrule
\end{tabular}
\end{table}

%% file: tables/competitor.tex
\begin{table}[!htb]
\caption{Localization and cost comparison. Top-2 recovery: features 0 and 4
occupy the top two ranks; N/A: no feature ranking; $^g$: derived from group-level
attributions. Centroids are in lags; windows count model-forward evaluations;
times use one GPU. Budgets: $m_{\pi}=3$ (cell-pShapley), $m_G=5$ (cell-ShapG).
All Shapley methods use the same mean baseline.}
\label{tab:competitor}
\centering\small
\setlength{\tabcolsep}{3pt}
\begin{tabular}{lcccccrr}
\toprule
Method & Shapley & \shortstack{Attribution\\granularity} & \shortstack{Top-2\\recovery} & \shortstack{Centroid $\bar t_0$\\(true: 9)} & \shortstack{Centroid $\bar t_4$\\(true: 17)} & windows & wall-clock \\
\midrule
feature-ShapG & yes & feature & Yes & --- & --- & $5.9$k & $1.3$\,s \\
WindowSHAP & yes & time chunk & N/A & --- & --- & $25$k & $0.2$\,s \\
GS-SHAP & yes & group$\times$segment & No$^g$ & $13.0^g$ & $11.5^g$ & $0.8$k & $64.2$\,s \\
PAX-TS & \textbf{no} & cell $(t,i)$ & Yes & $9.4$ & $15.2$ & $111$k & $0.7$\,s \\
cell-ShapG (graph) & yes & cell & Yes & $9.7$ & $16.3$ & $706$k & $23.6$\,s \\
\textbf{cell-pShapley} & yes & cell & Yes & $9.4$ & $16.4$ & $111$k & $\mathbf{0.5}$\,s \\
\bottomrule
\end{tabular}
\end{table}

%% file: tables/redundancy.tex
\begin{table}[!htb]
\caption{Share of the redundant pair's joint value
$\Delta_{\mathrm{pair}}$ credited by each method (synthetic; $10$ seeds,
mean\,$\pm$\,std, \%; cell-pShapley at $m_{\pi}{=}8$). The PAX-TS column
is its (non-additive) fraction of total attribution mass.}
\label{tab:redundancy}
\centering\small
\setlength{\tabcolsep}{8pt}
\begin{tabular}{lrrr}
\toprule
corr$(x_0,x_1)$ & cell-pShapley & LOO & PAX-TS mass \\
\midrule
$0.98$ & $104\pm8$ & $50\pm3$ & $53\pm5$ \\
$0.80$ & $109\pm10$ & $42\pm8$ & $43\pm5$ \\
$0.20$ & $72\pm47$ & $84\pm32$ & $46\pm5$ \\
\bottomrule
\end{tabular}
\end{table}

%% file: tables/cellkpi.tex
\begin{table}[!htb]
\caption{Weighted slope $S$ (mean\,$\pm$\,std over three model-training seeds)
of the deletion curve under mean masking;
bold = best mean per row, and $|\phi|$ is the
unsigned cell-pShapley ranking. The last two rows count how often each
ranking beats a baseline over the 27 (case, seed) runs; exact sign
tests give $p<10^{-6}$ (cell-pShapley vs cell-LOO) and $p<10^{-4}$
(vs PAX-TS).}
\label{tab:cellkpi}
\centering\small
\setlength{\tabcolsep}{3pt}
\resizebox{\linewidth}{!}{%
\begin{tabular}{llrrrrrr}
\toprule
Dataset & Family & cell-pShapley & $|\phi|$ & cell-ShapG & TimeSHAP-c & PAX-TS & cell-LOO \\
\midrule
Jena          & LSTM        & $1.12\pm0.19$ & $0.68\pm0.08$ & $\mathbf{1.16}\pm0.25$ & $1.12\pm0.19$ & $0.55\pm0.02$ & $0.63\pm0.07$ \\
Jena          & Transformer & $0.80\pm0.03$ & $0.62\pm0.03$ & $\mathbf{0.80}\pm0.03$ & $0.80\pm0.03$ & $0.52\pm0.02$ & $0.60\pm0.06$ \\
Jena          & DLinear     & $\mathbf{1.95}\pm0.20$ & $0.85\pm0.07$ & $1.93\pm0.20$ & $1.95\pm0.20$ & $0.81\pm0.09$ & $0.90\pm0.06$ \\
ETTh1         & LSTM        & $\mathbf{0.93}\pm0.01$ & $0.92\pm0.00$ & $0.92\pm0.01$ & $0.92\pm0.01$ & $0.91\pm0.01$ & $0.92\pm0.00$ \\
ETTh1         & Transformer & $\mathbf{0.92}\pm0.00$ & $0.91\pm0.01$ & $0.92\pm0.00$ & $0.91\pm0.01$ & $0.90\pm0.01$ & $0.90\pm0.00$ \\
ETTh1         & DLinear     & $\mathbf{1.00}\pm0.02$ & $0.93\pm0.00$ & $0.99\pm0.02$ & $1.00\pm0.02$ & $0.94\pm0.00$ & $0.93\pm0.00$ \\
Exchange-Rate & LSTM        & $0.94\pm0.09$ & $0.95\pm0.04$ & $\mathbf{0.99}\pm0.06$ & $0.94\pm0.09$ & $0.97\pm0.05$ & $0.83\pm0.09$ \\
Exchange-Rate & Transformer & $0.70\pm0.03$ & $\mathbf{0.71}\pm0.04$ & $0.70\pm0.04$ & $0.69\pm0.04$ & $0.67\pm0.12$ & $0.61\pm0.12$ \\
Exchange-Rate & DLinear     & $\mathbf{1.82}\pm0.13$ & $0.91\pm0.01$ & $1.78\pm0.12$ & $1.82\pm0.13$ & $0.98\pm0.00$ & $0.91\pm0.02$ \\
\midrule
\multicolumn{2}{l}{median $S$} & $\mathbf{0.932}$ & $0.901$ & $0.941$ & $0.930$ & $0.903$ & $0.876$ \\
\multicolumn{2}{l}{beats cell-LOO} & $26$ & $16$ & $24$ & $26$ & $16$ & --- \\
\multicolumn{2}{l}{beats PAX-TS} & $24$ & $14$ & $24$ & $24$ & --- & $10$ \\
\bottomrule
\end{tabular}}
\end{table}

%% file: tables/real.tex
\begin{table}[!htb]
\caption{Real-data \emph{when}/horizon-impact across families ($K=1$,
cell-pShapley $m_{\pi}{=}32$, mean baseline). \textbf{when-centroid} =
mean input lag weighted by $g(t)$ ($0$=oldest); \textbf{settle} = horizon step
at which the ranking reaches $\tau_h\ge0.9$ (single training runs).}
\label{tab:real}
\centering\small
\setlength{\tabcolsep}{6pt}
\begin{tabular}{llrlrr}
\toprule
Dataset & Family & test/naive & top feature & when-centroid & settle $\tau{\ge}0.9$ \\
\midrule
Jena          & LSTM        & $0.051$ & Tpot (K)     & $12.5$ & $16$ \\
Jena          & Transformer & $0.051$ & Tpot (K)     & $11.9$ & $17$ \\
Jena          & DLinear     & $0.012$ & rho (g/m**3) & $11.4$ & $2$ \\
ETTh1         & LSTM        & $0.048$ & OT           & $12.6$ & $5$ \\
ETTh1         & Transformer & $0.046$ & OT           & $13.0$ & $1$ \\
ETTh1         & DLinear     & $0.019$ & OT           & $12.2$ & $1$ \\
Exchange-Rate & LSTM        & $0.382$ & AUS          & $14.3$ & $1$ \\
Exchange-Rate & Transformer & $0.369$ & AUS          & $11.5$ & $4$ \\
Exchange-Rate & DLinear     & $0.034$ & AUS          & $11.5$ & $4$ \\
\bottomrule
\end{tabular}
\end{table}

%% file: tables/linear_gt.tex
\begin{table}[!htb]
\caption{Agreement with the closed-form DLinear oracle: signed Pearson,
per-cell $\tau$ and Spearman, and the max-abs centroid difference in lags.}
\label{tab:linear-gt}
\centering\small
\setlength{\tabcolsep}{7pt}
\begin{tabular}{llrrrr}
\toprule
Dataset & Estimator & signed Pearson & per-cell $\tau$ & Spearman & centroid $\Delta$ (lags) \\
\midrule
Jena          & \textbf{cell-pShapley} & $1.000$ & $1.000$ & $1.000$ & $0.00$ \\
              & cell-ShapG             & $0.986$ & $0.969$ & $0.980$ & $0.11$ \\
\addlinespace
ETTh1         & \textbf{cell-pShapley} & $1.000$ & $1.000$ & $1.000$ & $0.00$ \\
              & cell-ShapG             & $0.916$ & $0.904$ & $0.954$ & $0.50$ \\
\addlinespace
Exchange-Rate & \textbf{cell-pShapley} & $1.000$ & $1.000$ & $1.000$ & $0.00$ \\
              & cell-ShapG             & $0.871$ & $0.841$ & $0.872$ & $0.05$ \\
\bottomrule
\end{tabular}
\end{table}

%% file: tables/sensitivity.tex
\begin{table}[!htb]
\caption{Shift of the when-centroid, in lags, from the deployed setting of
each control run; the payoff columns keep the mean reference.
$^\dagger$ marks a variant that also changes which feature ranks first.}
\label{tab:sensitivity}
\centering\small
\setlength{\tabcolsep}{6pt}
\begin{tabular}{llrrrr}
\toprule
 & & \multicolumn{2}{c}{reference (vs mean)}
   & \multicolumn{2}{c}{payoff (vs error-reduction)} \\
\cmidrule(lr){3-4}\cmidrule(lr){5-6}
Dataset & Family & marginal & conditional & prediction & $R^2$ \\
\midrule
Jena          & LSTM        & $+1.4$ & $-0.6$ & $+0.2$ & $+1.2$ \\
Jena          & Transformer & $-0.1$ & $-0.7$ & $-0.1$ & $+0.7$ \\
Jena          & DLinear     & $+0.6$ & $+0.0$ & $-0.0$ & $-0.0$ \\
ETTh1         & LSTM        & $+0.4$ & $-1.0$ & $+0.2$ & $-0.5$ \\
ETTh1         & Transformer & $+0.2$ & $-0.4$ & $-0.3$ & $-0.9$ \\
ETTh1         & DLinear     & $+0.3$ & $+0.2$ & $-0.0$ & $+0.4$ \\
Exchange-Rate & LSTM        & $-0.1$ & $-3.0$ & $-0.0$ & $-0.4$ \\
Exchange-Rate & Transformer & $-0.5$ & $+0.3$ & $+0.1$ & $-0.0$ \\
Exchange-Rate & DLinear     & $-0.0$ & $+0.7$ & $+0.0$$^\dagger$ & $+0.3$ \\
\bottomrule
\end{tabular}
\end{table}

%% file: tables/budget.tex
\begin{table}[!htb]
\caption{Convergence of cell-pShapley against a high-budget reference of 128 antithetic pairs,
averaged over 3 sampling seeds; column groups are the estimation budget
$m_{\pi}$. The \emph{unpaired} group spends the same budget on independent
permutations instead of antithetic pairs: DLinear is exact only when paired.}
\label{tab:budget}
\centering\footnotesize
\setlength{\tabcolsep}{2pt}
\begin{tabular}{ll>{\raggedleft\arraybackslash}p{3.0em}>{\raggedleft\arraybackslash}p{3.0em}>{\raggedleft\arraybackslash}p{3.0em}>{\raggedleft\arraybackslash}p{3.0em}>{\raggedleft\arraybackslash}p{3.0em}>{\raggedleft\arraybackslash}p{3.0em}>{\raggedleft\arraybackslash}p{3.0em}>{\raggedleft\arraybackslash}p{3.0em}>{\raggedleft\arraybackslash}p{3.0em}}
\toprule
 & & \multicolumn{3}{c}{\shortstack{relative RMSE\\of $\phi$, paired}}
   & \multicolumn{3}{c}{\shortstack{relative RMSE\\of $\phi$, unpaired}}
   & \multicolumn{3}{c}{\shortstack{when-centroid\\error (lags)}} \\
\cmidrule(lr){3-5}\cmidrule(lr){6-8}\cmidrule(lr){9-11}
Dataset & Family & $1$ & $32$ & $64$ & $1$ & $32$ & $64$ & $1$ & $32$ & $64$ \\
\midrule
Jena          & LSTM        & $0.224$ & $0.042$ & $0.034$ & $0.601$ & $0.107$ & $0.074$ & $0.18$ & $0.01$ & $0.03$ \\
Jena          & Transformer & $0.194$ & $0.035$ & $0.025$ & $0.566$ & $0.094$ & $0.071$ & $0.09$ & $0.00$ & $0.01$ \\
Jena          & DLinear     & $0.000$ & $0.000$ & $0.000$ & $0.673$ & $0.120$ & $0.083$ & $0.00$ & $0.00$ & $0.00$ \\
ETTh1         & LSTM        & $0.064$ & $0.013$ & $0.010$ & $0.361$ & $0.052$ & $0.038$ & $0.10$ & $0.04$ & $0.02$ \\
ETTh1         & Transformer & $0.132$ & $0.026$ & $0.022$ & $0.387$ & $0.054$ & $0.043$ & $0.14$ & $0.02$ & $0.02$ \\
ETTh1         & DLinear     & $0.000$ & $0.000$ & $0.000$ & $0.301$ & $0.041$ & $0.033$ & $0.00$ & $0.00$ & $0.00$ \\
Exchange-Rate & LSTM        & $0.238$ & $0.049$ & $0.037$ & $0.600$ & $0.096$ & $0.078$ & $0.08$ & $0.04$ & $0.02$ \\
Exchange-Rate & Transformer & $1.753$ & $0.364$ & $0.259$ & $1.642$ & $0.397$ & $0.292$ & $0.59$ & $0.57$ & $0.26$ \\
Exchange-Rate & DLinear     & $0.000$ & $0.000$ & $0.000$ & $0.450$ & $0.090$ & $0.065$ & $0.00$ & $0.00$ & $0.00$ \\
\bottomrule
\end{tabular}
\end{table}

%% file: tables/onmanifold.tex
\begin{table}[!htb]
\caption{Mean-mask versus Gaussian conditional-mean
deletion. Entries report median $S$ and wins against each baseline over 27
(case, seed) runs; bold marks the highest median within each evaluator.}
\label{tab:onmanifold}
\centering\small
\setlength{\tabcolsep}{3pt}
\resizebox{\linewidth}{!}{%
\begin{tabular}{llrrrrrr}
\toprule
Evaluator & & cell-pShapley & $|\phi|$ & cell-ShapG & TimeSHAP-c & PAX-TS & cell-LOO \\
\midrule
mean mask    & median $S$    & $0.931$ & $0.901$ & $\mathbf{0.941}$ & $0.929$ & $0.902$ & $0.892$ \\
             & beats PAX-TS   & $24$ & $16$ & $24$ & $24$ & --- & $10$ \\
             & beats cell-LOO & $26$ & $16$ & $24$ & $26$ & $16$ & --- \\
\midrule
conditional  & median $S$    & $0.097$ & $\mathbf{0.269}$ & $0.054$ & $0.053$ & $0.212$ & $0.216$ \\
             & beats PAX-TS   & $3$ & $17$ & $3$ & $4$ & --- & $11$ \\
             & beats cell-LOO & $11$ & $21$ & $9$ & $10$ & $16$ & --- \\
\bottomrule
\end{tabular}}
\end{table}